%% file: main.tex
\documentclass[sigconf]{acmart}

\def\BibTeX{{\rm B\kern-.05em{\sc i\kern-.025em b}\kern-.08emT\kern-.1667em\lower.7ex\hbox{E}\kern-.125emX}}
    
\copyrightyear{2022}
\acmYear{2022}
\setcopyright{rightsretained}

\acmConference[KDD '22]{Proceedings of the 28th ACM SIGKDD Conference on Knowledge Discovery and Data Mining}{August 14--18, 2022}{Washington, DC, USA}
\acmBooktitle{Proceedings of the 28th ACM SIGKDD Conference on Knowledge Discovery and Data Mining (KDD '22), August 14--18, 2022, Washington, DC, USA}
\acmISBN{978-1-4503-9385-0/22/08}
\acmDOI{10.1145/3534678.3539478}

\usepackage{etoolbox}
\makeatletter
\patchcmd{\maketitle}{\@copyrightpermission}{
   \begin{minipage}{0.3\columnwidth}
     \href{https://creativecommons.org/licenses/by/4.0/}{\includegraphics[width=0.90\textwidth]{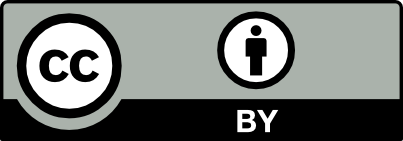}}
   \end{minipage}\hfill
   \begin{minipage}{0.7\columnwidth}
     \href{https://creativecommons.org/licenses/by/4.0/}{This work is licensed under a Creative Commons Attribution International 4.0 License.}
   \end{minipage}

   \vspace{5pt}
}{}{}

\makeatother

\usepackage{makecell}

\usepackage{amsmath}
\usepackage{subfig}
\usepackage{hhline}
\usepackage{multicol}
\usepackage{mathtools}
\usepackage{tabularx}
\usepackage{paralist}
\usepackage{bm}
\usepackage{balance}
\usepackage{multirow}
\usepackage{amsthm}
\usepackage{graphicx}
\usepackage{psfrag}
\usepackage{blindtext}
\usepackage{color}
\usepackage{url}
\usepackage{amsfonts}
\usepackage{enumerate}
\usepackage{booktabs} % Allows the use of \toprule, \midrule and \bottomrule in tables
\usepackage{geometry}
\usepackage{marginnote}
\usepackage{algorithm}
\usepackage[noend]{algpseudocode}
\usepackage{todonotes}
\algnewcommand\algorithmicinput{\textbf{Input:}}
\algnewcommand\algorithmicoutput{\textbf{Output:}}
\algnewcommand\Input{\item[\algorithmicinput]}%
\algnewcommand\Output{\item[\algorithmicoutput]}%

\input{dfn.tex}

\begin{document}

\title{Geometric Policy Iteration for Markov Decision Processes}

\author{Yue Wu}
\affiliation{%
  \institution{University of California, Davis}
  \city{}
  \state{}
  \country{}
}
\email{yvwu@ucdavis.edu}

% \authornotemark[1]
\author{Jes\'us A. De Loera}
\affiliation{%
  \institution{University of California, Davis}
  \city{}
  \state{}
  \country{}
}
\email{deloera@math.ucdavis.edu}

\begin{abstract}
Recently discovered polyhedral structures of the value function for finite discounted Markov decision processes (MDP) shed light on understanding the success of reinforcement learning. We investigate the value function polytope in greater detail and characterize the polytope boundary using a hyperplane arrangement. We further show that the value space is a union of finitely many cells of the same hyperplane arrangement, and relate it to the polytope of the classical linear programming formulation for MDPs. Inspired by these geometric properties, we propose a new algorithm, \emph{Geometric Policy Iteration} (GPI), to solve discounted MDPs. GPI updates the policy of a single state by switching to an action that is mapped to the boundary of the value function polytope, followed by an immediate update of the value function. This new update rule aims at a faster value improvement without compromising computational efficiency. Moreover, our algorithm allows asynchronous updates of state values which is more flexible and advantageous compared to traditional policy iteration when the state set is large. We prove that the complexity of GPI achieves the best known bound $\bigO{\frac{|\actions|}{1 - \gamma}\log \frac{1}{1-\gamma}}$ of policy iteration and empirically demonstrate the strength of GPI on MDPs of various sizes.
\end{abstract}




\maketitle

%%%%%%%%%%%%%%%%%%%%%%%%%%%%%%%%%%%%%%%%%%%%%%%%%%%%%%
\section{Introduction}
\label{sec:introduction}

The Markov decision process~(MDP) is the mathematical foundation of reinforcement learning~(RL) which has achieved great empirical success in sequential decision problems. Despite RL's success, new mathematical properties of MDPs are to be discovered to better theoretically understand RL algorithms. In this paper, we study the geometric properties of discounted MDPs with finite states and actions, and propose a new value-based algorithm inspired by their polyhedral structures. 

A large family of methods for solving MDPs is based on the notion of (state) values. The strategy of these methods is to maximize the values, then extract the optimal policy from the optimal values. One classic method is value iteration~\cite{howard60dynamic, Bertsekas1987_VI} in which values are greedily improved to optimum using the Bellman operator. It is also well known that the optimal values can be solved by linear programming~(LP)~\cite{puterman94markov} which attracts a lot of research interest due to its mathematical formulation. The most efficient algorithms in practice are often variants of policy iteration~\cite{howard60dynamic} which facilitates the value improvement with policy updates. The \textbf{value function}, which maps policies to the value space, is central to our analysis throughout, and it plays a key role in understanding how values are related to policies from a geometric perspective.

Although policy iteration and its variants are very efficient in practice, their worst-case complexity was long believed exponential~\cite{mansour1999complexity}. The major breakthrough was made by \citet{ye2011} where the author proved that both policy iteration and LP with Simplex method~\cite{danzigsimplex} terminate in $\bigO{\frac{|\states||\actions|}{1-\gamma}\log \frac{|\states|}{1-\gamma}}$. The author first proved that the Simplex method with the most-negative-reduced-cost pivoting rule is strongly polynomial in this situation. Then, a variant of policy iteration called simple policy iteration was shown to be equivalent to the Simplex method. \citet{hansen2013strategy} later improved the complexity of policy iteration by a factor of $|S|$. The best known complexity of policy iteration is $\bigO{\frac{|\actions|}{1-\gamma}\log \frac{1}{1-\gamma} }$ proved by \citet{scherrer2016improved}.

In the LP formulation, the state values are optimized through the vertices of the LP feasible region which is a convex polytope. Surprisingly, it was recently discovered that the space of the value function is a (possibly non-convex) polytopes~\cite{Dadashi2019value}. We call such object the \textbf{value function polytope} denoted by $\valuespace$. As opposed to LP, the state values are navigated through $\valuespace$ in policy iteration. Moreover, the \emph{line theorem}~\cite{Dadashi2019value} states that the set of policies that only differ in one state is mapped onto the same line segment in the value function polytope. This suggests the potential of new algorithms based on single-state updates. 

Our first contribution is on the structure of the value function polytope $\valuespace$. Specifically, we show that a hyperplane arrangement $H_{MDP}$ is shared by $\valuespace$ and the polytope of the linear programming formulation for MDPs. We characterize these hyperplanes using the Bellman equation of policies that are deterministic in a single state. We prove that the boundary of the value function polytope $\partial \valuespace$ is the union of finitely many (convex polyhedral) cells of $H_{MDP}$. Moreover, each full-dimensional cell of the value function polytope is contained in the union of finitely many full-dimensional cells defined by $H_{MDP}$. We further conjecture that the cells of the arrangement cannot be partial, but they have to be entirely contained in the value function polytope.

The learning dynamic of policy iteration in the value function polytope shows that every policy update leads to an improvement of state values along one line segment of $\valuespace$. Based on this, we propose a new algorithm, \textbf{geometric policy iteration}~(GPI), a variant of the classic policy iteration with several improvements. First, policy iteration may perform multiple updates on the same line segment. GPI avoids this situation by always reaching an endpoint of a line segment in the value function polytope for every policy update. This is achieved by efficiently calculating the true state value of each potential policy update instead of using the Bellman operator which only guarantees a value improvement. Second, GPI updates the values for all states immediately after each policy update for a single state, which makes the value function monotonically increasing with respect to every policy update. Last but not least, GPI can be implemented in an asynchronous fashion. This makes GPI more flexible and advantageous over policy iteration in MDPs with a very large state set. 

We prove that GPI converges in $\bigO{\frac{|\actions|}{1-\gamma}\log \frac{1}{1-\gamma}}$ iterations, which matches the best known bound for solving finite discounted MDPs. Although using a more complicated strategy for policy improvement, GPI maintains the same $\bigO{|\states|^2|\actions|}$ arithmetic operations in each iteration as policy iteration. We empirically demonstrate that GPI takes fewer iterations and policy updates to attain the optimal value.

\subsection{Related Work}

One line of work related to this paper is on the complexity of the policy iteration. For MDPs with a fixed discount factor, the complexity of policy iteration has been improved significantly~\cite{Littman94, ye2011, Ye2013Post, hansen2013strategy, scherrer2016improved}. There are also positive results reported on stochastic games (SG). \citet{hansen2013strategy} proved that a two-player turn-based SG can be solved by policy iteration in strongly polynomial time when the discount factor is fixed. \citet{Akian2013PolicyIF} further proved that policy iteration is strongly polynomial in mean-payoff SG with state-dependent discount factors under some restrictions. In terms of more general settings, the worst-case complexity can still be exponential~\cite{mansour1999complexity, Fearnley, Hollanders2012, Hollanders2016}. Another line of related work studies the geometric properties of MDPs and RL algorithms. The concept of the value function polytope in this paper was first proposed in \citet{Dadashi2019value}, which was also the first recent work studying the geometry of the value function. Later, \citet{Bellemare2019Geometric} explored the direction of using these geometric structures as auxiliary tasks in representation learning in deep RL. \citet{policyimprovepath} also aimed at improving the representation learning by shaping the policy improvement path within the value function polytope. The geometric perspective of RL also contributes to unsupervised skill learning where no reward function can be accessed~\cite{unsupervised_skill_learning}. Very recently, \citet{geometryPOMDP} analyzed the geometry of state-action frequencies in partially observable MDPs, and formulated the problem of finding the optimal memoryless policy as a polynomial program with a linear objective and polynomial constraints. The geometry of the value function in robust MDP is also studied in \citet{geometryRMDP}.

% The rest of the paper is organized as follows. Section~\ref{sec:prelim} introduces the fundamentals of the MDP. We show the relation of the value function polytope $\valuespace$ and LP polytope as well as the analysis of the boundary of $\valuespace$ in Section~\ref{sec:boundary}. The detail of GPI is presented in Section~\ref{sec:gpi}.

%%%%%%%%%%%%%%%%%%%%%%%%%%%%%%%%%%%%%%%%%%%%%%%%%%%%%%
\section{Preliminaries}

An MDP has five components $\mathcal{M} = \langle\states , \actions, \rewards, \transitions, \gamma\rangle$ where 
        $\states$ and $\actions$ are finite state set and action set,
        $\transitions: \states \times \actions \to \Delta(\states)$ is the transition function with $\Delta(\cdot)$ denoting the probability simplex.
        $\rewards: \states \times \actions \to \realset$ is the reward function and
        $\gamma = [0,1)$ is the discount factor that represents the value of time.

A policy $\pi: \states \to \Delta(\actions)$ is a mapping from states to distributions over actions. The goal is to find a policy that maximizes the cumulative sum of rewards.

Define $V^{\pi} \in \realset^{|\states|}$ as the vector of state values. $V^{\pi}(s)$ is then the expected cumulative reward starting from a particular state $s$ and acting according to $\pi$:
\begin{equation*}
V^{\pi}(s) = \E\nolimits_{P^\pi}\Big(\sum^{\infty}_{i=0} \gamma^{i}
\rewards(s_i, a_i) \cbar s_0=s \Big). \label{eq:v_func_cumulative_reward}
\end{equation*}

The \textbf{Bellman equation}  \cite{Bellman:DynamicProgramming} connects the value $V^{\pi}$ at a state $s$ with the value at the subsequent states when following $\pi$:
\begin{align} V^\pi(s) = \E\nolimits_{P^\pi}\Big(\rewards(s, a) + \gamma V^\pi(s')\Big).
\label{eq:bellman}
\end{align}
Define $\rpi$ and $\Ppi$ as follows.
\begin{align*}
    \rpi(s) &= \sum_{a \in \actions} \pi(a \cbar s) \rewards(s, a), \\
    \Ppi(s' \cbar s) &= \sum_{a \in \actions} \pi(a \cbar s) \transitions (s' \cbar s, a),
\end{align*}
% where $\actions_s$ denotes the set of actions available in $s$.
Then, the Bellman equation for a policy $\pi$ can be expressed in matrix form as follows.
\begin{align}
    V^\pi &= \rpi + \gamma \Ppi V^\pi \nonumber\\
    & = (I - \gamma \Ppi)^{-1} \rpi. \label{eq:bellman_eq}
\end{align}
% The space of policies $\policies$ is the Cartesian product of simplices that we can express as a space of $|\states| \times |\actions|$ matrices. 

Under this notation, we can define the Bellman operator $\bellop^{\pi}$ and the optimality Bellman operator $\bellop^*$ for an arbitrary value vector $V$ as follows.
\begin{align*}
    \bellop^{\pi} V &= \rpi + \gamma \Ppi V, \\
    \bellop^* V &= \max_{\pi} \bellop^{\pi} V.
\end{align*}
$V$ is optimal if and only if $V = \bellop^* V$. MDPs can be solved by \textbf{value iteration}~(VI)~\cite{Bellman:DynamicProgramming} which consists of the repeated application of the optimality Bellman operator $V^{(k+1)} := \bellop^* V^{(k)}$ until a fixed point has been reached.

Let $\policyspace$ denote the space of all policies, and $\valuespace$ denote the space of all state values. We define the \textbf{value function} $f_v(\pi): \policyspace \to \valuespace$ as
\begin{equation}
    f_v(\pi) = (I - \gamma \Ppi)^{-1} \rpi.
    \label{eq:value_func_mapping}
\end{equation}
The value function $\valuefunction$ is fundamental to many algorithmic solutions of an MDP. \textbf{Policy iteration}~(PI)~\cite{howard60dynamic} repeatedly alternates between a policy evaluation step and a policy improvement step until convergence. In the policy evaluation step, the state values $\Vpi$ of the current policy $\pi$ is evaluated which involves solving a linear system (Eq.~\eqref{eq:bellman_eq}). In the policy improvement step, PI iterates over all states and update the policy by taking a greedy step using the optimality Bellman operator as follows.
\begin{equation*}
    \pi'(s) \in \argmax_{a \in \actions}\left\{ \rewards(s, a) + \gamma \sum_{s'}\transitions(s' \cbar s, a) \Vpi(s') \right\},\, \forall s \in \states.
\end{equation*}
\textbf{Simple policy iteration}~(SPI) is a variant of policy iteration. It only differs from policy iteration in the policy improvement step where the policy is only updated for the state-action pair with the largest improvement over the following advantage function.
\begin{equation*}
    \tilde{A}(s, a) = \rewards(s, a) + \gamma \sum_{s'}\transitions(s' \cbar s, a) \Vpi(s') - \Vpi(s). 
\end{equation*}
SPI selects a state-action pair from $\argmax_{s, a} \tilde{A}(s, a)$ then updates the policy accordingly. 

\subsection{Geometry of the Value Function}
While the space of policies $\policies$ is the Cartesian product of $|\states|$ probability simplices, \citet{Dadashi2019value} proved that the value function space is a possibly non-convex polytope~\cite{Ziegler_polytope}. Figure~\ref{fig:vfp_line} shows a convex and a non-convex $f_v$ polytopes of 2 MDPs in blue regions. The proof is built upon the line theorem which is an equally important geometric property of the value space. The line theorem depends on the following definition of policy determinism.

\begin{definition}[Policy Determinism]
A policy $\pi$ is 
\begin{itemize}
    \item \emph{$s$-deterministic} for $s\in \states$ if it selects one concrete action for sure in state $s$, i.e., $\pi(a|s)\in \{0, 1\},\, \forall a$;
    \item \emph{deterministic} if it is $s$-deterministic for all $s \in \states$.
\end{itemize}
\end{definition}

\begin{figure}[h]
    \centering
    \captionsetup[subfloat]{farskip=0pt,captionskip=-1pt}
    \subfloat[\label{fig:vfp_line_left}]{\includegraphics[width=0.51\linewidth]{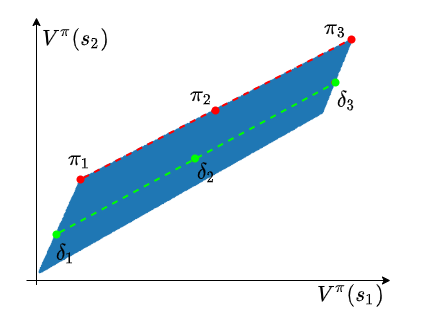}} \hspace{-1.2em}
    \subfloat[\label{fig:vfp_line_right}]{\includegraphics[width=0.51\linewidth]{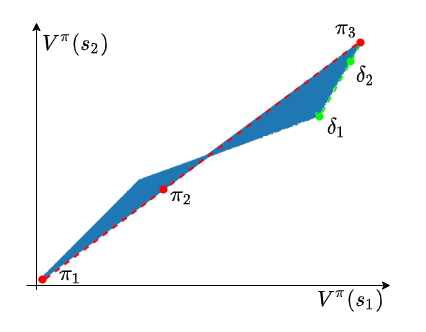}}\\[-2ex]
    \caption{The blue regions are the value spaces of 2 MDPs with $|\states|=2$ and $|\actions|=2$. The regions are obtained by plotting $f_v$ of $50,000$ random policies. (a): Both $\{\pi_i\}$ and $\{\delta_i\}$ agree on $s_1$ but differ in $s_2$. $\pi_1$ and $\pi_3$ are deterministic. $\pi_2$ is $s_1$-deterministic. $\delta_1$ and $\delta_3$ are $s_2$-deterministic. (b): $\{\pi_i\}$ and $\{\delta_i\}$ agree on $s_1$ and $s_2$, respectively. $\pi_1$, $\pi_3$, and $\delta_1$ are deterministic while $\pi_2$ and $\delta_2$ are $s_1$ and $s_2$-deterministic, respectively.}
    \label{fig:vfp_line}
\end{figure}

The line theorem captures the geometric property of a set of policies that differ in only one state. Specifically, we say two policies $\pi_1, \pi_2$ \emph{agree} on states $s_1,.., s_k \in \states$ if 
$\pi_1(\cdot \cbar s_i) = \pi_2(\cdot \cbar s_i)$ for each $s_i$, $i = 1, \dots, k$.
For a given policy $\pi$, we denote by $Y^\pi_{s_1, \dots, s_k} \subseteq \policies$ the set of policies that agree with $\pi$ on 
$s_1, \dots, s_k$; we will also write $\agreeone$ to describe the set of policies that agree with $\pi$ on all states except $s$. When we keep the probabilities fixed at all but state $s$, the functional $f_v$ draws a line segment which is oriented in the positive orthant (that is, one end dominates the other). Furthermore, the endpoints of this line segment are $s$-deterministic 
policies. 

The line theorem is stated as follows:
\begin{theorem}[Line theorem~\cite{Dadashi2019value}]
\label{thm:line}
Let $s$ be a state and $\pi$ a policy.
Then there are two $s\text{-deterministic}$ policies in $\agreeone$, denoted $\pi_l, \pi_u$, which
bracket the value of all other policies $\pi' \in \agreeone$: \begin{equation*}
    f_v(\pi_l) \preccurlyeq f_v(\pi') \preccurlyeq f_v(\pi_u).
\end{equation*}
\end{theorem}

For both Figure~\ref{fig:vfp_line_left} and~\ref{fig:vfp_line_right}, we plot policies that agree on one state to illustrate the line theorem. The policy determinism decides if policies are mapped to a vertex, onto the boundary or inside the polytope.

%%%%%%%%%%%%%%%%%%%%%%%%%%%%%%%%%%%%%%%%%%%%%%%%%%%%%%
\section{The Cell Structure of the Value Function Polytope}
\label{sec:vfp_boundary}

% \begin{figure}[t]
%     \captionsetup[subfloat]{labelformat=empty}
%     \centering
%     \subfloat[]{\includegraphics[width=0.5\linewidth]{imgs/qe_vfp_boundary1.png}%
%     }
%     \hfil
%     \subfloat[]{\includegraphics[width=0.5\linewidth]{imgs/qe_vfp_boundary2.png}%
%     }
% \end{figure}

In this section, we revisit the geometry of the (non-convex) value function polytope presented in \citet{Dadashi2019value}. We establish a connection to linear programming formulations of the MDP which then can be adapted to show a finer description of cells in the value function polytope as unions of cells of a hyperplane arrangement. For more on hyperplane arrangements and their structure, see \citet{hyperplanes-intro}.

It is known since at least the 1990's that finding the optimal value function of an MDP can be formulated as a linear program (see for example \cite{puterman94markov, bertsekas96neurodynamic}). In the primal form, the feasible constraints are defined by $\{ V \in \reals^{|\states|} \; \big| \; V \gvec \bellop^*  V \}$, where $\bellop^*$ is the optimality Bellman operator. Concretely, the following linear program is well-known to be equivalent 
to maximizing the expected total reward in Eq.~\eqref{eq:v_func_cumulative_reward}. 
We call this convex polyhedron the \textbf{MDP-LP polytope} (because it is a linear programming form of the MDP problem).
    \begin{align*}
        \min_V & \quad \sum_{s} \alpha(s) V(s) \\
        \text{s.t.}  & \quad V(s) \ge \rewards(s, a) + \gamma \sum_{s'}\transitions(s' \cbar s, a) V(s'),\,\, \forall s \in \states, a \in \actions.
    \end{align*}
where $\alpha$ is a probability distribution over $\states$. 

Our main new observation is that the MDP-LP polytope and the value polytope are actually closely related, and one
can describe the regions of the (non-convex) value function polytope in terms of the (convex) cells of the arrangement.

\begin{theorem} Consider the hyperplane arrangement $H_{MDP}$, with $|\actions| |\states|$ hyperplanes,
consisting of those of the MDP polytope, i.e.,
\begin{equation*}
     H_{MDP}= \left\{V(s) = \rewards(s, a) + \gamma \sum_{s'}\transitions (s' \cbar s, a)V(s') \mid \forall s \in \states, a \in \actions \right\}.
\end{equation*}
\label{thm: vfp_boundary}
Then, the boundary of the value function polytope $\partial \valuespace$ is the union of finitely (convex polyhedral) cells of the arrangement $H_{MDP}$. Moreover, each full-dimensional cell of the value polytope is contained in the union of finitely many full-dimensional cells defined by $H_{MDP}$.
\end{theorem}

\begin{proof}
%We can prove this by induction on the dimension of the cells of the value polytope.  
%Now assume by the induction hypothesis that all cells of dimension $k$ in the value polytope are union of $k$-cells of the arrangement $H_{MDP}$. 

Let us first consider a point $V^\pi$ being on the boundary of the value function polytope. Theorem 2 and Corollary 3 of \citet{Dadashi2019value}
demonstrated that the boundary of the space of value functions is a (possibly proper) subset of the ensemble of value functions of policies, where at least one state has a fixed deterministic choice for all actions.
Note that from the value function Eq.~\eqref{eq:value_func_mapping},  then the hyperplane 
\begin{equation*}
    V(s) = \rewards(s, a^s_l) + \gamma \sum_{s'}\transitions (s' \cbar s, a^s_l)V(s')
\end{equation*}
includes all policies taking policy $a^s_l = \pi_l(s)$ in state $s$. Thus the points of the boundary of the value function polytope are contained in the hyperplanes of $H_{MDP}$. Now we can see how the $k$-dimensional cells of the boundary are then in the intersections of the hyperplanes too.

%\begin{equation*}
%     \left\{V(s) =\rewards(s, a^s_l) + \gamma \sum_{s'}\transitions (s' \cbar s, a^s_l)V(s') \cbar \forall a \in \actions \right\}
%\end{equation*}

The zero-dimensional cells (vertices) are clearly a subset of the zero-dimensional cells of the arrangement $H_{MDP}$ because, by above results, the zero-dimensional cells are precisely in the intersection of $| \states |$ many hyperplanes from $H_{MDP}$, which is equivalent to choosing a fixed set of actions for all states. This corresponds to solving a linear system consisting of the hyperplanes that bound $\valuespace$ (same as Eq.~\eqref{eq:bellman_eq}).
But more generally, if we fix the policies for only $k$ states, the induced space lies in a $|\states|-k$ dimensional affine space. Consider a policy $\pi$ and $k$ states $s_1, \dots, s_k$, and write $C_{k+1}^\pi, \dots, C_{|\states|}^\pi$ for the columns of the matrix $(I - \gamma \Ppi)^{-1}$ corresponding to states \emph{other} than $s_1, \dots, s_k$. Define the affine vector space $\affinesev$
\begin{equation*}
    H^\pi_{s_1, \dots, s_k} = V^\pi + Span(C_{k+1}^\pi, \dots, C_{|\states|}^\pi) .
\end{equation*}
Now For a given policy $\pi$, we denote by $Y^\pi_{s_1, \dots, s_k} \subseteq \policies$ the set of policies which agree with $\pi$ on $s_1, \dots, s_k$;
Thus the value functions generated by $\agree$ are contained in the affine vector space $H^\pi_{s_1, \dots, s_k}$:
$    f_v(\agree) = \valuespace \cap \affinesev.$

The points of $H^\pi_{s_1, \dots, s_k}$ in one or more of the $H_{MDP}$ planes (each hyperplane is precisely fixing one policy action pair). This is the intersection of $k$ hyperplanes given by the following equations. 
\begin{equation*}
     \left\{V(s) = \rewards(s, a) + \gamma \sum_{s'}\transitions (s' \cbar s, a)V(s') \mid \forall s \in \{s_1,\ldots,s_k\}, a \in \actions \right\}. 
\end{equation*}
Thus we can be sure of the stated containment.

Finally, the only remaining case is when $V^\pi$ is in the interior of the value polytope. If that is the case, because $H_{MDP}$ partitions the entire Euclidean space, it must be contained in at least one of the full-dimensional cell of $H_{MDP}$. 
\qedhere
\end{proof}

\begin{figure}[h]
    \centering
    \captionsetup[subfloat]{farskip=0pt,captionskip=-1pt}
    \subfloat[\label{fig:vfp_lp_boundary}]{\includegraphics[width=0.45\linewidth, height=3.5cm]{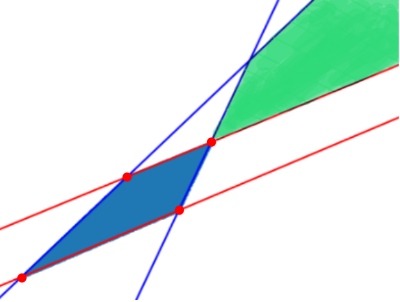}} \hspace{-1.2em}
    \subfloat[\label{fig:vfp_boundary}]{\includegraphics[width=0.53\linewidth, height=3.5cm]{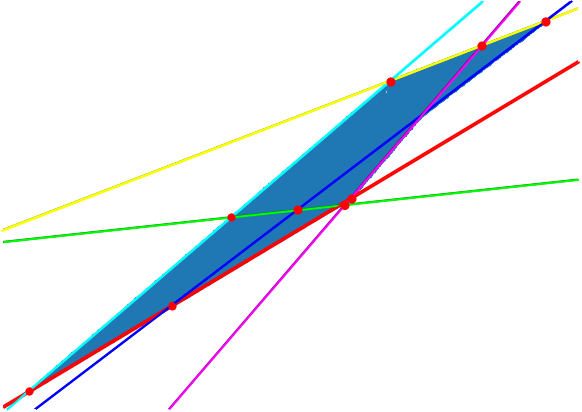}}\\[-2ex]
    \caption{(a): $f_v$ polytope (blue) and MDP-LP polytope (green) of an MDP with $|\states|=2$ and $|\actions|=2$. (b) $f_v$ polytope overlapped with the hyperplane arrangement $H_{MDP}$ from Theorem~\ref{thm: vfp_boundary}. This MDP has 3 actions so $|H_{MDP}|=6$.}
\end{figure}

Figure \ref{fig:vfp_lp_boundary} is an example of the value function polytope in blue, MDP-LP polytope in green and its bounding hyperplanes (the arrangement $H_{MDP}$) as blue and red lines. In Figure~\ref{fig:vfp_boundary} we exemplify Theorem \ref{thm: vfp_boundary} by presenting a value function polytope with delimited boundaries where  $H_{MDP}$ hyperplanes are indicated in different colors. The deterministic policies are those for which $\pi(a|s) \in \{0, 1\} \, \forall a\in \actions, s\in \states$.  In both pictures, the values of deterministic policies in the value space are shown as red dots. The boundaries of the value polytope are indeed included in the set of cells of the arrangement $H_{MDP}$  as stated by Theorem~\ref{thm: vfp_boundary}.
These figures of value function polytopes (blue regions) were obtained by randomly sampling policies and plotting their corresponding state values.   

% \begin{figure}[ht]
%     \centering
%     \includegraphics[width=\linewidth]{imgs/vfp_lp_polytope.png}
%     \caption{Value function polytopes~(blue regions) and the corresponding LP feasible region~(green region). Deterministic policies are shown as red dots. Both MDPs have 2 states and 2 actions, 3 actions for each state in (a) and (b), respectively.}
%     \label{fig:vfp_lp}
% \end{figure}

%  \begin{figure}[ht]
%      \centering
%     \includegraphics[width=0.6\linewidth]{imgs/vpolytope_with_boundary_multicolor.png}
%      \caption{Value function polytope overlapped with the hyperplane arrangement $H_{MDP}$ from Theorem~\ref{thm: vfp_boundary}. }
%      \label{fig:vfp_boundary}
%  \end{figure}

Some remarks are in order. Note how sometimes the several adjacent cells of the MDP arrangement together form 
a connected cell of the value function polytope. We also observe that for any set of states $s_1,.., s_k \in \states$ and a policy $\pi$, $V^{\pi}$ can be expressed as a convex combination of value functions of $\{s_1,.., s_k\}$-deterministic policies. In particular,  $\valuespace$ is included in the convex hull of the value functions of deterministic policies. It is also demonstrated clearly in Figure~\ref{fig:vfp_boundary} that the value functions of deterministic policies are not always vertices and the vertices of the value polytope are not always value functions of deterministic policies, but they are always intersections of hyperplanes on $H_{MDP}$. 
%Figure~\ref{fig:vfp_lp} (b) shows that deterministic policies are not necessarily mapped to vertices of the value function polytope. 
However, optimal values will always include a deterministic vertex. This observation suggests that it would suffice to find the optimal policy by only visiting deterministic policies on the boundary. It is worthwhile to note that the optimal value of our MDP would be at the unique intersection vertex of the two polytopes. We note that the blue regions in Figure~\ref{fig:vfp_lp_boundary} are \emph{not} related to the polytope of the dual formulation of LP. Unlike the MDP polytope which can be characterized as the intersection of finitely many half-spaces, we do not have such a neat representation for the value function polytope. The pictures presented here and many more experiments we have done suggest the following stronger result is true:

{\bf Conjecture:} if the value polytope intersects a cell of the arrangement $H_{MDP}$, then it contains the entire cell, thus all full-dimensional cells of the value function polytope are equal to the union of full-dimensional cells of the arrangement.

Proving this conjecture requires showing that the map from policies to value functions is surjective over the cells it touches. At the moment we can only guarantee that there are no isolated components because the value polytope is a compact set. More strongly \citet{Dadashi2019value} shown (using the line theorem) that there is  path connectivity from $V^{\pi}$, in any cell, to others is guaranteed by a polygonal path. More precisely if we let $V^{\pi}$ and $V^{\pi'}$ be two value functions. Then there exists a sequence of $k \le |\states|$ policies, $\pi_1, \dots, \pi_k$, such that $V^\pi = V^{\pi_1}$, $V^{\pi'} = V^{\pi_k}$, and for every $i \in 1, \dots, k - 1$, the set $\{ f_v( \alpha \pi_i + (1 - \alpha) \pi_{i+1}) \cbar \alpha \in [0, 1] \}$ forms a line segment.

%From the line theorem, there are two $s\text{-deterministic}$ policies in $\agreeone$, denoted $\pi_l, \pi_u$, which
%bracket the value of all other policies $\pi' \in \agreeone$: 
%\begin{equation*}
%    f_v(\pi_l) \preccurlyeq f_v(\pi') \preccurlyeq f_v(\pi_u).
%\end{equation*}

It was observed that algorithms for solving MDPs have different learning behavior when visualized in the value polytope space. 
For example, policy gradient methods~\cite{sutton2000policy, kakade2002natural, policygradient_actorcritic, policygradientWilliam, policygradientWilliamsPeng91} have an improvement path inside of the value function polytope; value iteration can go outside of the polytope which means there can be no corresponding policy during the update process; and policy iteration navigates exactly through deterministic policies. 
In the rest of our paper we use this geometric intuition to design a new algorithm.

%%%%%%%%%%%%%%%%%%%%%%%%%%%%%%%%%%%%%%%%%%%%%%%%%
 
\section{The Method of Geometric Policy Iteration}
\label{sec:gpi}

\begin{figure}[h]
    \centering
    \captionsetup[subfloat]{farskip=0pt,captionskip=-1pt}
    \subfloat[]{\includegraphics[width=0.51\linewidth]{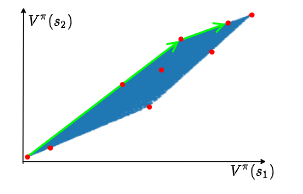}} \hspace{-1.2em}
    \subfloat[]{\includegraphics[width=0.51\linewidth]{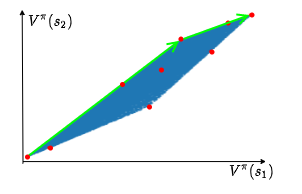}}\\[-2ex]
    \caption{The value sequences of one iteration which involves a sweep over all states looking for policy updates. (a): In PI, we may not reach the end of a line segment for an action switch. (b): An endpoint is always reached in GPI.}
    \label{fig:gpiv_endpoints}
\end{figure}

We now present \emph{geometric policy iteration} (GPI) that improves over PI based on the geometric properties of the learning dynamics. Define an \textit{action switch} to be an update of policy $\pi$ in any state $s\in \states$. The Line theorem shows that policies agreeing on all but one state lie on a line segment. So an action switch is a move along a line segment to improve the value function. In PI, we use the optimality Bellman operator $\bellop^* V(s) = \max_{\pi}( r^{\pi} + \gamma P^{\pi} V)(s)$ to decide the action to switch to for state $s$. However, $\bellop^* V (s)$ does not guarantee the largest value improvement $V^*(s) - V(s)$ for $s$. This phenomenon is illustrated in Figure~\ref{fig:gpiv_endpoints} where we plot the value sequences of PI and the proposed GPI. 

% Algebraically, for any two policies $\pi'$ and $\pi$, the difference between these two methods can be characterized as follows.
% \begin{equation}
%     \Vpiprime - \Vpi = \left(I - \gamma \Ppiprime\right)^{-1} \left( \bellop^{\pi'}\Vpi - \Vpi \right).
% \end{equation}
% The proof can be found in Section~\ref{sec:gpi_theory}.

We propose an alternative action-switch strategy in GPI that directly calculates the improvement of the value function for one state. By choosing the action with the largest value improvement, we can always reach the endpoint of a line segment which potentially reduces the number of action switches.

This strategy requires efficient computation of the value function because a naive calculation of the value function by Eq.~\eqref{eq:bellman_eq} is very expensive due to the matrix inversion. On the other hand, PI only re-evaluates the value function once per iteration. Our next theorem states that the new state-value can be efficiently computed. This is achieved by using the fact that the policy improvement step can be done state-by-state within a sweep over the state set, so adjacent policies in the update sequence only differ in one state.

% Our main contribution here is to use the Sherman-Morrison formula to efficiently calculate the best action $a_i$ for state $s_i$ and then update the value function for all states. 

\begin{theorem}
\label{thm:sherman_morrison}
 Given $\bQ^\pi = (I - \gamma \Ppi)^{-1}$ and $V^\pi = \bQ^\pi r^\pi$. If a new policy $\delta$ only differs from $\pi$ in state $s$ with $\delta(s) = a \ne \pi(s)$, $\Vdelta(s)$ can be calculated efficiently by
    \begin{equation}
        \Vdelta(s) = \parentheses{ \bone_s + \frac{\bQ^\pi(s, s)}{1 - \bw^\top_a \,\bq_s} \, \bw_a}^{\top} \parentheses{\Vpi + \Delta r_a \, \bq_s },
        \label{eq:sherman_morrison_thm}
    \end{equation}
where $\bw_a = \gamma \left(\transitions(s, a) - \transitions(s, \pi(s))\right)$ is a $|\states|$-d vector, $\Delta r_a = \rewards(s, a) - \rewards(s, \pi(s))$ is a scalar, $\bq_s$ is the $s^{\text{th}}$ column of $\bQ^\pi$, and $\bone_s$ is a vector with entry $s$ being $1$, others being $0$.

%  where $\bq_s$ is the $s$-th column of $\bQ^\pi$,
%  \begin{align*}
%      \bw_a &= \gamma \left(\transitions(s, \delta(s), :) - \transitions(s, \pi(s), :)\right),\\
%      \Delta r_a &= \rewards(s, \delta(s)) - \rewards(s, \pi(s)),
%  \end{align*}
% and $\bone_s$ is a vector with entry $s$ being $1$, others being $0$.
\end{theorem}

\begin{proof}
We here provide a general proof that we can calculate $V^{\delta}$ given policy $\pi$, $\Vpi$, and $\delta$ differs from $\pi$ in only one state.
\begin{align*}
    \Vdelta = \left( I - \gamma \Pdelta \right)^{-1} \rdelta = \parentheses{ I - \gamma \Ppi - \gamma \Delta P}^{-1} \rdelta,
\end{align*}
where $\Delta P = \Pdelta - \Ppi$. Assume $\delta$ and $\pi$ differ in state $s$. $\Delta P$ is a rank-$1$ matrix with row $j$ being $\transitions(s, \pi(s)) - \transitions(s, a)$, and all other rows being zero vectors.

We can then express $\Delta P$ as the outer product of two vectors $\Delta P = \bone_{s} \bw_a^\top$, where $\bone_{s}$ is a one-hot vector
\begin{align}
    \bone_{s}(i) = 
    \begin{cases}
        1, & \text{if $i = s$,} \\
        0, & \text{otherwise,}
    \end{cases}
\end{align}
and $\bw_a$ is defined above.
% \begin{equation}
%     \bw_a = \gamma \left(\transitions(s, a) - \transitions(s, \pi(s))\right).
% \end{equation}

Similarly, we have $\rdelta = \rpi + \Delta r = \rpi + \Delta r_a \bone_s$.
Then, we have
\begin{align*}
        \Vdelta & = \parentheses{ I - \gamma \Pdelta}^{-1} \rdelta \\ 
        & = \parentheses{ I - \gamma \Ppi - \gamma \Delta P}^{-1} \parentheses{\rpi + \Delta r} \\
        & = \parentheses{ I - \gamma \Ppi - \bone_{s} \bw_a^\top}^{-1} \parentheses{\rpi + \Delta r_a \, \bone_s}  \\
        & = \parentheses{ \bQ^\pi + \frac{\bQ^\pi \bone_s \bw_a^\top \bQ^\pi}{1 - \bw_a^\top \bQ^\pi \bone_s} } \parentheses{\rpi + \Delta r_a \, \bone_s} \tag*{$\parentheses{\text{Sherman-Morrison, }\bQ^\pi = \parentheses{I - \gamma \Ppi}^{-1}}$} \\
        % & = \parentheses{ I + \frac{\bQ \bone_s \bw^\top }{1 - \bw^\top \bQ \bone_s} } \parentheses{\bQ \rpi + \Delta r(s) \, \bQ \, \bone_s} \\
        & = \parentheses{ I + \frac{\bq_s \bw_a^\top }{1 - \bw_a^\top \bq_s} } \parentheses{\Vpi + \Delta r_a \, \bq_s} \tag*{$\parentheses{\text{here, }\bQ^\pi \bone_s = \bq_s ,\,\bQ^\pi \rpi = \Vpi}$}.
    \end{align*}
Thus, for state $s$, we have 
\begin{equation*}
    \Vdelta(s) = \parentheses{\bone_s + \frac{\bQ^\pi(s, s)}{1 - \bw_a^\top \bq_s} \bw_a} \parentheses{\Vpi + \Delta r_a \, \bq_s},
\end{equation*}
which completes the proof.
\end{proof}

Theorem~\ref{thm:sherman_morrison} suggests that updating the value of a single state using Eq.~\eqref{eq:sherman_morrison_thm} takes $\bigO{|\states||\actions|}$ arithmetic operations which matches the complexity of the optimality Bellman operator used in policy iteration.

\begin{figure}[h]
    \centering
    \captionsetup[subfloat]{farskip=0pt,captionskip=-1pt}
    \subfloat[]{\includegraphics[width=0.51\linewidth]{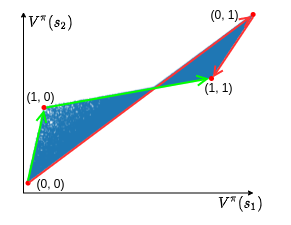}} \hspace{-1.2em}
    \subfloat[]{\includegraphics[width=0.51\linewidth]{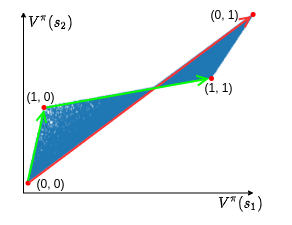}}\\[-2ex]
    \caption{Two paths are shown for each PI, GPI. The green and red paths denote one iteration with $\pi(s_1)$ and $\pi(s_2)$ updated first, respectively. (a): The policy improvement path of PI. The red path is not action-switch-monotone which will lead to an additional iteration. (b): GPI is always action-switch-monotone. The red path achieves the optimal values in one action switch.}
    \label{fig:gpiv_monotone}
\end{figure}

The second improvement over policy iteration comes from the fact that the value improvement path in $\valuespace$ may not be monotonic with respect to action switches. Although it is well-known that the update sequence $\{\Vpik\}$ is non-decreasing in the iteration number $k$, the value function could decrease in the policy improvement step of the policy iteration. An illustration is shown as the red path of PI in Figure~\ref{fig:gpiv_monotone}. The possible value decrease is because when the Bellman operator $\bellop$ is used to decide an action switch, $V$ is fixed for the entire sweep of states. This leads us to the motivation for GPI which is to update the value function after each action switch such that the value function action-switch-monotone. This idea can be seamlessly combined with Theorem~\ref{thm:sherman_morrison} since the values of all states can be updated efficiently in $\bigO{|\states|^2|\actions|}$ arithmetic operations. Thus, the complexity of completing one iteration is the same as policy iteration. 
\begin{algorithm}
\caption{Geometric Policy Iteration}
\begin{algorithmic}[1]
    \Input $\transitions$, $\rewards$, $\gamma$
    \State set iteration number $k = 0$ and randomly initialize $\piki{k}{0}$
    \State Calculate $\bQki{k}{0} = (I - \gamma \Pki{k}{0})^{-1}$ and $\Vki{k}{0} = \bQki{k}{0} \rki{k}{0}$ \label{alg:gpi_line_initial_eval}
    % \item[\textbf{3.}] For all $s \in \states$
    \For{$i=1, \ldots, |\states|$}  \label{alg:gpi_line_for}
        \State calculate the best action $a_i$ according to Eq.~\eqref{eq:gpiv_action_selection}
        \label{alg:gpi_line_action_select}
        \State update $\bQki{k}{i}$ according to Eq.~\eqref{eq:gpi_Q_update} \label{alg:gpi_line_Q_update}
        \State $\piki{k}{i}(i) = a_i$ \label{alg:gpi_line_pi_update}
        \State $\Vki{k}{i} = \bQki{k}{i} \, \rki{k}{i}$ \label{alg:gpi_line_V_update}
    \EndFor
    \If{$\Vki{k}{|\states|}$ is optimal}
    % \State \textbf{return} $\piki{k}{i}$
    \Return $\piki{k}{|\states|}$
    \EndIf
    \State $\bQki{k+1}{0} = \bQki{k}{|\states|}$, $\piki{k+1}{0} = \piki{k}{|\states|}$, $\Vki{k+1}{0} = \Vki{k}{|\states|}$, $k = k + 1$. Go to step~\ref{alg:gpi_line_for}
\end{algorithmic}
\label{alg:gpi}
\end{algorithm}

We summarize GPI in Algorithm~\ref{alg:gpi}. GPI looks for action switches for all states in one iteration, and updates the value function after each action switch. Let superscript $k$ denote the iteration index, subscript $i$ denote the state index in one iteration. To avoid clutter, we use $i$ to denote the state $s_i$ being updated and drop superscript $\pi$ in $\Ppi$ and $\rpi$.  Step~\ref{alg:gpi_line_initial_eval} evaluates the initial policy $\piki{k}{0}$. The difference here is that we store the intermediate matrix $\bQki{k}{0}$ for later computation. From step~\ref{alg:gpi_line_for} to step~\ref{alg:gpi_line_V_update}, we iterate over all states to search for potential updates. In step~\ref{alg:gpi_line_action_select}, GPI selects the best action by computing the new state-value of each potential action switch by Eq.~\eqref{eq:gpiv_action_selection}.
\begin{equation}
    a_i \in \argmax_{a \in \actions}\left\{\parentheses{ \bone_i + \frac{\bQki{k}{i-1}(i, i)}{1 - \bw^\top_a \,\bqi} \, \bw_a}^{\top} \parentheses{\Vki{k}{i-1} + \Delta r_a \, \bqi } \right\},
    \label{eq:gpiv_action_selection}
\end{equation}
where
\begin{align}
    \bw_a &= \gamma \left(\transitions(i, a) - \transitions(i, \piki{k}{i-1}(i))\right),\label{eq:gpiv_w} \\ 
    \Delta r_a &= \parentheses{\rewards(i, a) - \rewards(i, \piki{k}{i-1}(i))},\label{eq:gpiv_r}
\end{align}
and $\bone_i$ is a vector with $i^{\text{th}}$ entry being $1$ and others being $0$.

Define $\bqi$ to be the $i^{\text{th}}$ column of $\bQki{k}{i-1}$. $\bwiT$ is obtained by Eq.~\eqref{eq:gpiv_w} using the selected action $a_i$. In step~\ref{alg:gpi_line_Q_update}, we update $\bQki{k}{i}$ as follows.
\begin{equation}
    \bQki{k}{i} = \bQki{k}{i-1} + \frac{\bqi \, \bwiT \bQki{k}{i-1}}{1 - \bwiT \bqi}. \label{eq:gpi_Q_update}
\end{equation}
The policy is updated in step~\ref{alg:gpi_line_pi_update} and the value vector is updated in step~\ref{alg:gpi_line_V_update} where $\rki{k}{i}$ is the reward vector under the new policy. The algorithm is terminated when the optimal values are achieved.

\subsection{Theoretical Guarantees}
\label{sec:gpi_theory}

Before we present any properties of GPI, let us first prove the following very useful lemma.

\begin{lemma}
Given two policies $\pi$ and $\pi'$, we have the following equalities.
\begin{align}
    \Vpiprime - \Vpi &= \left(I - \gamma \Ppiprime\right)^{-1} \left( \rpiprime + \gamma \Ppiprime \Vpi - \Vpi \right)  \label{eq:lemma_adv_v},\\
    \Vpiprime - \Vpi &= \left(I - \gamma \Ppi \right)^{-1} \left( \Vpiprime - \rpi - \gamma \Ppi \Vpiprime \right) \label{eq:lemma_v_adv}.
\end{align}
\label{lemma:v_adv}
\end{lemma}
\begin{proof}
Using Bellman equation, we have
\begin{equation}
    \Vpiprime - \Vpi = \rpiprime + \gamma \Ppiprime \Vpiprime - \rpi - \gamma \Ppi \Vpi. \label{eq:lemme_1_expand}
\end{equation}
Eq.~\eqref{eq:lemme_1_expand} can be rearranged as
\begin{equation*}
    \Vpiprime - \Vpi = \rpiprime - \rpi + \gamma \Ppiprime \left( \Vpiprime - \Vpi \right) + \gamma \left(\Ppiprime - \Ppi \right) \Vpi,
    \label{eq:lemma_adv_v_proof}
\end{equation*}
and Eq.~\eqref{eq:lemma_adv_v} follows.

To get Eq.~\eqref{eq:lemma_v_adv}, we rearrange Eq.~\eqref{eq:lemme_1_expand} as
\begin{equation*}
    \Vpiprime - \Vpi = \rpiprime - \rpi + \gamma \Ppi \left( \Vpiprime - \Vpi \right) + \gamma \left(\Ppiprime - \Ppi \right) \Vpiprime, \label{eq:lemma_v_adv_proof}
\end{equation*}
and Eq.~\eqref{eq:lemma_v_adv} follows.
\end{proof}

Our first result is an immediate consequence of re-evaluating the value function after an action switch.
\begin{proposition}
The value function is non-decreasing with respect to action switches in GPI, i.e., $\Vki{k}{i+1} \ge \Vki{k}{i}$.
\label{proposition:action_switch_monotone}
\end{proposition}

\begin{proof}
From Eq.~\eqref{eq:lemma_adv_v} in Lemma~\ref{lemma:v_adv}, we have 
% \begin{align}
%     \Vpiprime - \Vpi &= \left(I - \gamma \Ppiprime\right)^{-1} \left( \rpiprime + \gamma \Ppiprime \Vpi - \Vpi \right)  \label{eq:lemma_adv_v},\\
%     \Vpiprime - \Vpi &= \left(I - \gamma \Ppi \right)^{-1} \left( \Vpiprime - \rpi - \gamma \Ppi \Vpiprime \right) \label{eq:lemma_v_adv}.
% \end{align}
\begin{equation*}
    \left(I - \gamma \Ppiprime\right)\left(\Vpiprime - \Vpi\right) = \rpiprime + \gamma \Ppiprime \Vpi - \Vpi.
\end{equation*}
Since $\Ppi \ge 0$, we have 
\begin{equation*}
    \Vpiprime - \Vpi \ge \rpiprime + \gamma \Ppiprime \Vpi - \Vpi = \bellop^{\pi'} \Vpi - \Vpi,
\end{equation*}
which implies that for any $\pi'$, $\pi$,
\begin{equation}
    \Vpiprime \ge \bellop^{\pi'} \Vpi. \label{eq:v_ge_tv}
\end{equation}
Now, consider $\piki{k}{i+1}$ and $\piki{k}{i}$. According to the updating rule of GPI, for state $i$ we have $\Vki{k}{i+1}(i) \ge \Vki{k}{i}(i)$. For state $j \ne i$, we have
\begin{equation*}
    \Vki{k}{i+1}(j) \ge \Tki{k}{i+1}\Vki{k}{i}(j) = \Tki{k}{i}\Vki{k}{i}(j) = \Vki{k}{i}(j).
\end{equation*}
Combined, we have $\Vki{k}{i+1} \ge \Vki{k}{i}$, which completes the proof.
\end{proof}

We next turn to the complexity of GPI and bound the number of iterations required 
to find the optimal solution. The analysis depends on the lemma described as follows.

\begin{lemma}
Let $\Vstar$ denote the optimal value. At iteration $k$ of GPI, we have the following inequality.
    \begin{equation*}
        \left(\Vstar - \Vki{k}{i}\right)(i) \le \gamma \Pstar \left(\Vstar - \Vki{k-1}{i} \right)(i).
    \end{equation*}
    \label{lemma:v_diff_opt_k}
\end{lemma}
\begin{proof}
From Bellman equation, we have
    \begin{align*}
        \Vstar - \Vki{k}{i} & = \Tstar \Vstar - \Vki{k}{i} \\
        & = \Tstar \Vstar - \Tstar \Vki{k-1}{i} + \Tstar \Vki{k-1}{i} - \Vki{k}{i}.
    \end{align*}
    For state $i$, we have
    \begin{align}
        & \left(\Vstar - \Vki{k}{i}\right)(i) \\ 
        & = \gamma \Pstar \left(\Vstar - \Vki{k-1}{i} \right)(i) + \left(\Tstar \Vki{k-1}{i} - \Vki{k}{i}\right)(i) \nonumber \\
        & \le \gamma \Pstar \left(\Vstar - \Vki{k-1}{i} \right)(i) + \left(\max_\pi \bellop^\pi \Vki{k-1}{i} - \Vki{k}{i} \right)(i)\nonumber \\
        & \le \gamma \Pstar \left(\Vstar - \Vki{k-1}{i} \right)(i) + \left(\Vki{k}{i} - \Vki{k}{i}\right)(i) \label{eq:proof_lemma_ineq} \\
        & =  \gamma \Pstar \left(\Vstar - \Vki{k-1}{i} \right)(i).
        \nonumber
    \end{align}
    % \eqref{eq:proof_gpi_update_ineq} is because of the updating rule of GPI,
    % \begin{align*}
    %     \Tstar \Vki{k-1}{i} (i) & = \left(\rstar + \gamma \Pstar \Vki{k-1}{i} \right)(i) \\
    %     & \le \max_{\pi} \left(\rpi + \gamma \Ppi \Vki{k-1}{i} \right)(i) \\
    %     & = \Tki \Vki{k-1}{i} (i).
    % \end{align*}
    Let $\pi' = \argmax_\pi \bellop^\pi \Vki{k-1}{i}$. The inequality \eqref{eq:proof_lemma_ineq} is because of the updating rule of GPI and \eqref{eq:v_ge_tv},
    % \begin{align*}
    %     \left( \bellop^{\pi'} \Vki{k-1}{i} - \Vki{k-1}{i} \right) & = \left(I - \gamma P^{\pi'} \right) \left( \Vki{k}{i} - \Vki{k-1}{i} \right) \\
    %     & \le \Vki{k}{i} - \Vki{k-1}{i},
    % \end{align*}
    \begin{equation*}
        \Vki{k}{i}\ge V^{\pi'}_i \ge \bellop^{\pi'}\Vki{k-1}{i},
    \end{equation*}
    which completes the proof.
\end{proof}

\begin{theorem}
 GPI finds the optimal policy in $\bigO{\frac{|\actions|}{1-\gamma}\log \frac{1}{1-\gamma}}$ iterations.
\end{theorem}

\begin{proof}
Define $\Dk \in \realset^{|\states|}$ with $\Dk(s) = \Vstar(s) - \Vki{k}{i}(s),\, \forall s\in \states$. Then, by Lemma~\ref{lemma:v_diff_opt_k}, we have 
\begin{align*}
    \Dk & \le \gamma \Pstar \Dkmo, \\
    \norminf{\Dk} & \le \gamma^k  \norminf{\Dz}.
\end{align*}
Let $j$ be the state such that $\Dz(j) = \norminf{\Dz}$, the following properties can be obtained by Eq.~\eqref{eq:lemma_v_adv} in Lemma~\ref{lemma:v_adv}.
\begin{align*}
    \norminf{\Dk} &\le \gamma^k  \norminf{\Dz} \\
    & \le \gamma^k \left\lVert\parentheses{I - \gamma \Pki{0}{j}}^{-1}\right\rVert_{\infty} \left\lVert\Vstar - \Tki{0}{j} \Vstar \right\rVert_{\infty} \\
    & = \frac{\gamma^k}{1 - \gamma} \left( \Vstar - \Tki{0}{j} \Vstar \right)(j).
\end{align*}
Also from Eq.~\eqref{eq:lemma_v_adv}, we have 
\begin{equation}
    \parentheses{\Vstar - \Tki{k}{j} \Vstar}(j) \le \Dk(j) \le \norminf{\Dk}.
\end{equation}
It follows that
\begin{align*}
    \left(\Vstar - \Tki{k}{j} \Vstar \right) (j) \le \frac{\gamma^k}{1 - \gamma} \left( \Vstar - \Tki{0}{j} \Vstar \right)(j),
\end{align*}
which implies when $\frac{\gamma^k}{1 - \gamma} < 1$, the non-optimal action for $j$ in $\pi^{(0)}$ is switched in $\pi^{(k)}$ and will never be switched back to in future iterations.
Now we are ready to bound $k$. By taking the logarithm for both sides of $\frac{\gamma^k}{1 - \gamma} < 1$, we have
\begin{align*}
    k \log \gamma & \ge \log (1 - \gamma) \\
    k & > \frac{\log(1 - \gamma)}{\log \gamma} = \frac{\log \frac{1}{1 - \gamma}}{\log \frac{1}{\gamma}} \\
    k & > \frac{1}{1 - \gamma}\log \frac{1}{1 - \gamma} \;\;\;\; \left(\log\frac{1}{\gamma} \ge \frac{\frac{1}{\gamma} - 1}{\frac{1}{\gamma}} = 1 - \gamma\right).
\end{align*}
Each non-optimal action is eliminated after at most $\bigO{\frac{1}{1 - \gamma}\log \frac{1}{1 - \gamma}}$ iterations, and there are $\bigO{|\actions|}$ non-optimal actions. Thus, GPI takes at most $\bigO{\frac{|\actions|}{1 - \gamma}\log \frac{1}{1 - \gamma}}$ iterations to reach the optimal policy.
\end{proof}

\begin{figure*}[t]
    \captionsetup[subfloat]{labelformat=empty}
    \centering
    \subfloat[]{\includegraphics[width=0.20\linewidth, height=3.1cm]{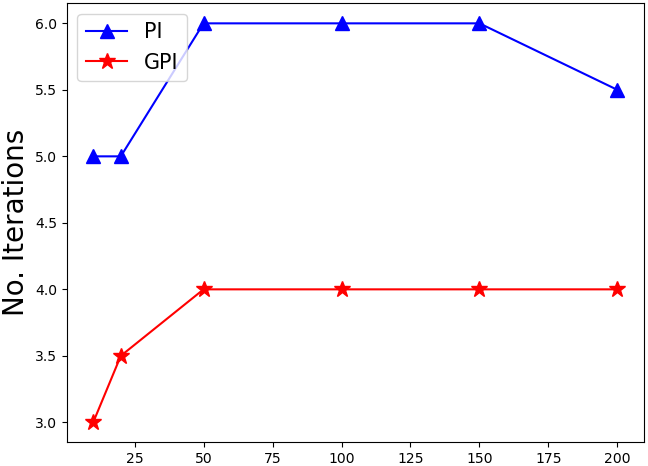}}
    \hfil
    \subfloat[]{\includegraphics[width=0.20\linewidth, height=3.1cm]{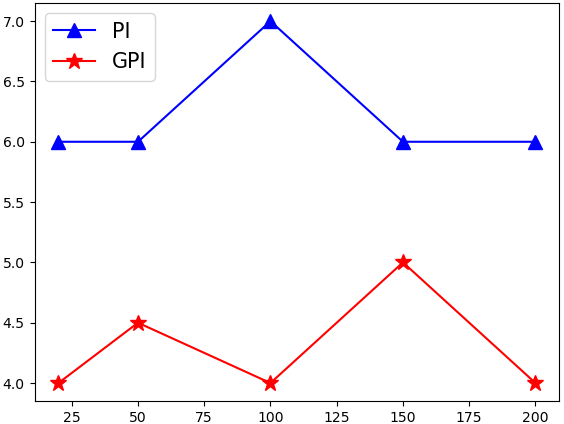}}
    \hfil
    \subfloat[]{\includegraphics[width=0.20\linewidth, height=3.1cm]{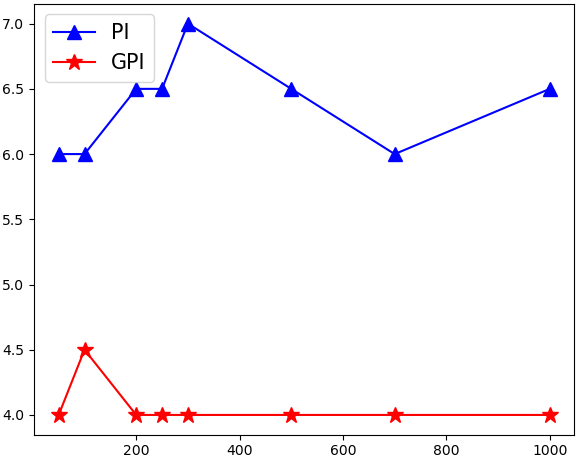}}
    \hfil
    \subfloat[]{\includegraphics[width=0.20\linewidth, height=3.1cm]{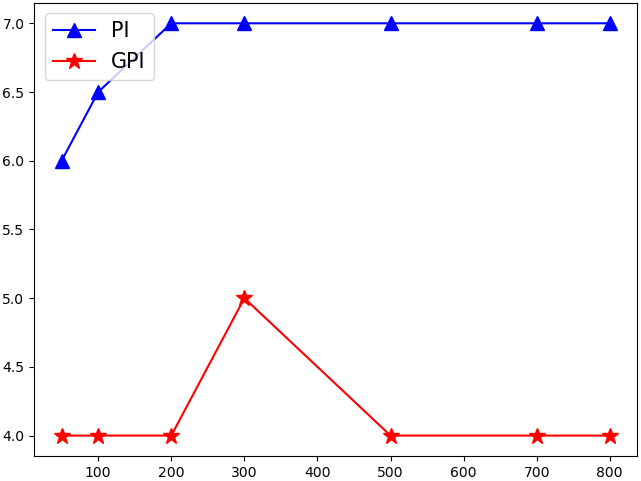}}
    \hfil
    \subfloat[]{\includegraphics[width=0.20\linewidth, height=3.1cm]{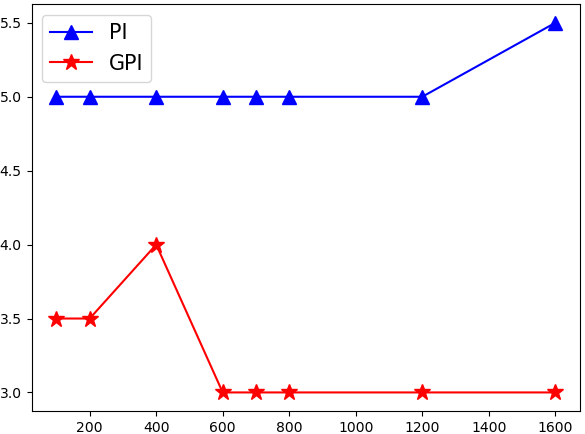}}
    \vspace{-7ex}
\end{figure*}
\begin{figure*}[t]
    \captionsetup[subfloat]{labelformat=empty}
    \centering
    \subfloat[]{\includegraphics[width=0.20\linewidth, height=3.1cm]{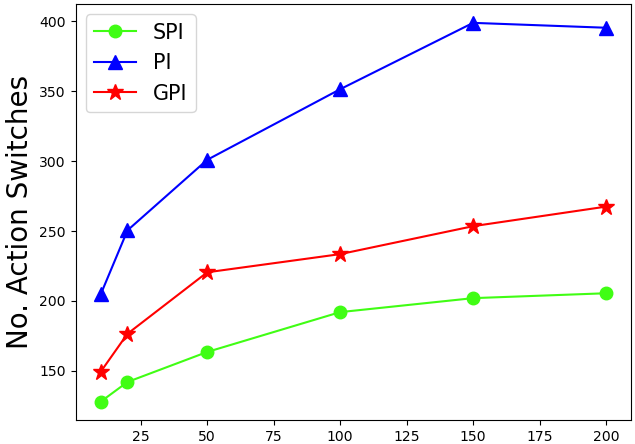}}
    \hfil
    \subfloat[]{\includegraphics[width=0.20\linewidth, height=3.1cm]{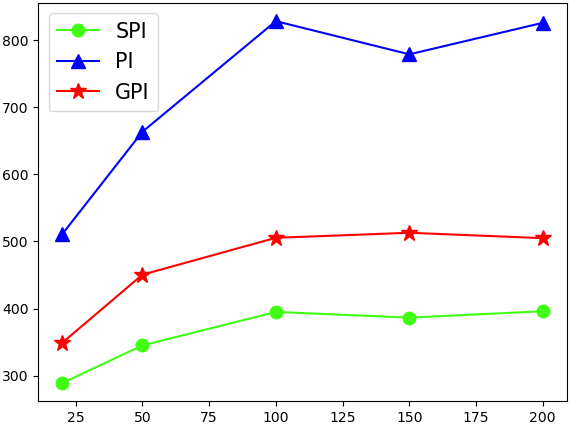}}
    \hfil
    \subfloat[]{\includegraphics[width=0.20\linewidth, height=3.1cm]{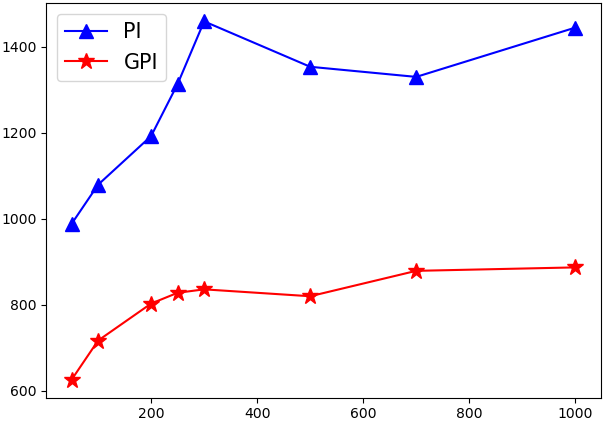}}
    \hfil
    \subfloat[]{\includegraphics[width=0.20\linewidth, height=3.1cm]{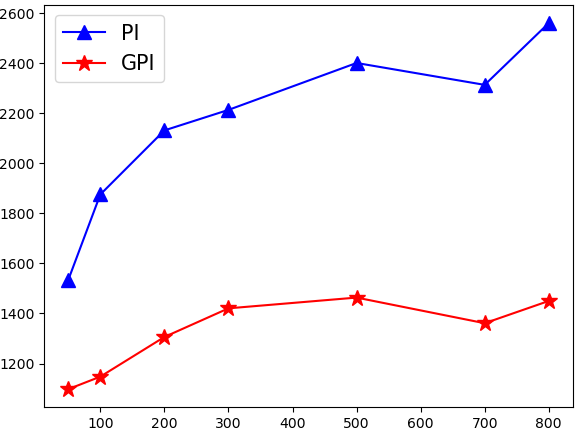}}
    \hfil
    \subfloat[]{\includegraphics[width=0.20\linewidth, height=3.1cm]{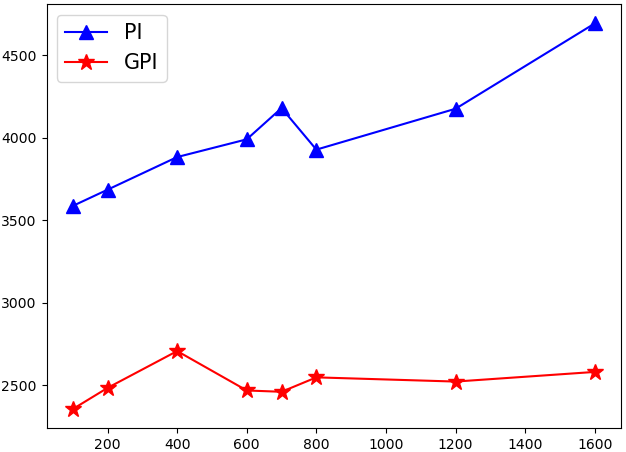}}
    \vspace{-7ex}
\end{figure*}

\begin{figure*}[t]
    % \captionsetup[subfloat]
    \centering
    \subfloat[]{\includegraphics[width=0.20\linewidth, height=3.1cm]{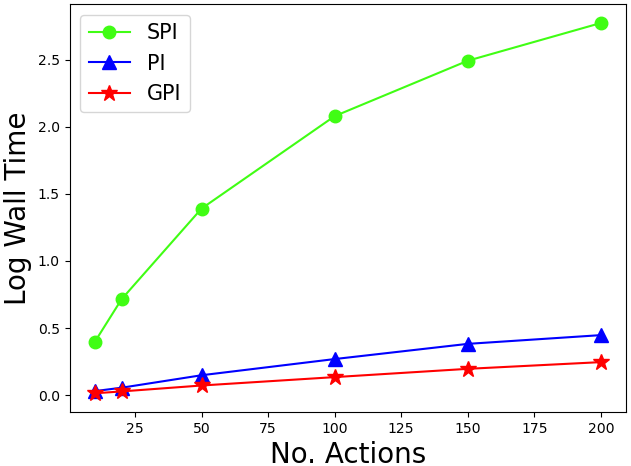}}
    \hfil
    \subfloat[]{\includegraphics[width=0.20\linewidth, height=3.1cm]{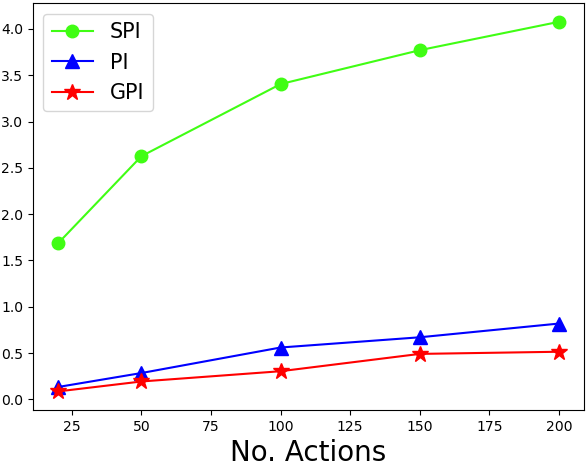}}
    \hfil
    \subfloat[]{\includegraphics[width=0.20\linewidth, height=3.1cm]{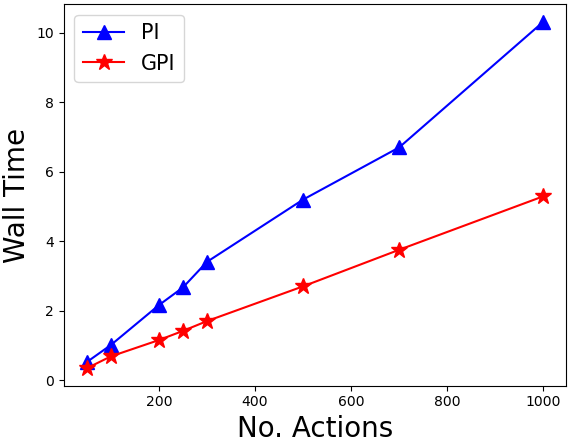}}
    \hfil
    \subfloat[]{\includegraphics[width=0.20\linewidth, height=3.1cm]{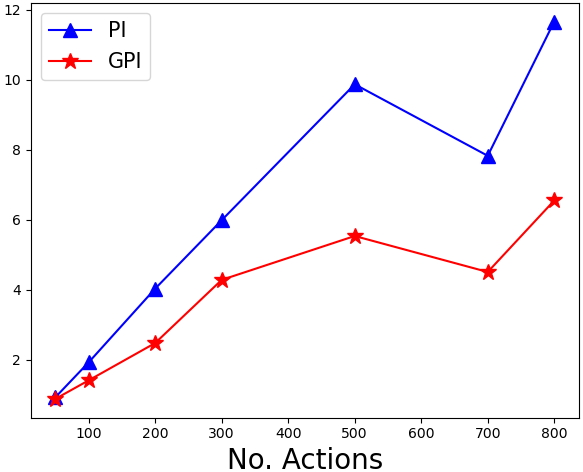}}
    \hfil
    \subfloat[]{\includegraphics[width=0.20\linewidth, height=3.1cm]{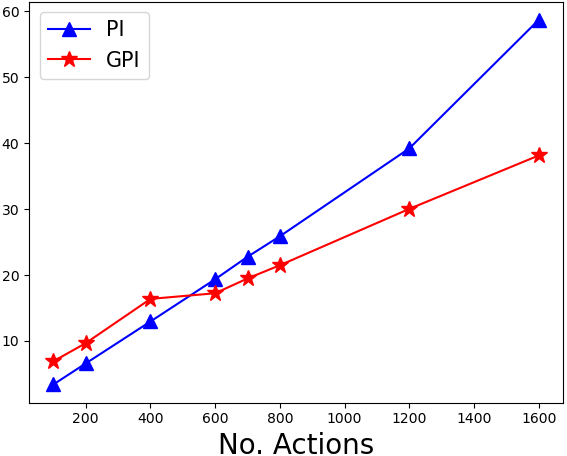}}
    \caption{The results of MDPs with $\{100, 200, 300, 500, 1000\}$ states in (a)-(e). The horizontal axes are the number of actions for all graphs. The vertical axes are the number of iterations, number of action switches, and wall time for the first to the third row, respectively. The performance curves of SPI, PI, and GPI are in green, blue, and red, respectively. The SPI curves are only presented in (a) and (b) to provide a ``lower bound" on the number of action switches, and are dropped for larger MDPs due to its higher running time. The number of switches of GPI remains low compared to PI. The proposed GPI consistently outperforms PI in both iteration count and wall time. The advantages of GPI become more significant as the action set size grows.}
    \label{fig:results_1}
\end{figure*}

\subsection{Asynchronous Geometric Policy Iteration}
When the state set is large it would be beneficial to perform policy updates in an orderless way~\cite{Sutton1998}. This is because iterating over the entire state set may be prohibitive, and exactly evaluating the value function with Eq.~\eqref{eq:bellman_eq} may be too expensive. Thus, in practice, the value function is often approximated when the state set is large. One example is modified policy iteration~\cite{puterman_modified_pi, puterman94markov} where the policy evaluation step is approximated with certain steps of value iteration. 

Since GPI avoids the matrix inversion by updating the value function incrementally, it has the potential to update the policy for arbitrary states available to the agent. This property also opens up the possibility of asynchronous (orderless) updates of policies and the value function when the state set is large or the agent has to update the policy for the state it encounters in real-time. The asynchronous update strategy can also help avoid being stuck in states that lead to minimal progress and may reach the optimal policy without reaching a certain set of states. 

Asynchronous GPI (Async-GPI) follows the action selection mechanism of GPI, and its general framework is as follows. Assume the transition matrix is available to the agent, we randomly initialize the policy $\pi^{(0)}$ and calculate the initial $\bQ^{(0)}$ and $V^{(0)}$ accordingly. In real-time settings, the sequence of states $\{s_0, s_1,s_2, \ldots \}$ are collected by an agent through real-time interaction with the environment. At time step $t$, we search for an action switch for state $s_t$ using Eq.~\eqref{eq:gpiv_action_selection}. Then, we update the $\pi^{(t)}$, $\bQ^{(t)}$ with Eq.~\eqref{eq:gpi_Q_update}, and $V^{(t)}$. Asynchronous value-based methods converge if each state is visited infinitely often~\cite{bertsekasAsyncVI}. We later demonstrate in experiments that Async-GPI converges well in practice.

\begin{figure*}[t]
    \centering
    \subfloat[]{\includegraphics[width=0.25\linewidth, height=3.2cm]{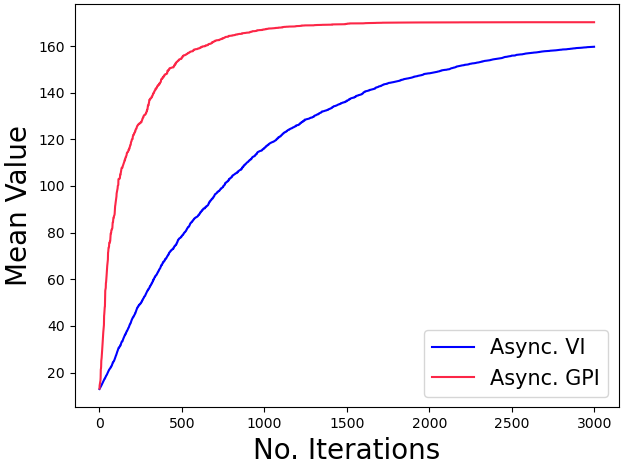}}
    \hfil
    \subfloat[]{\includegraphics[width=0.25\linewidth, height=3.2cm]{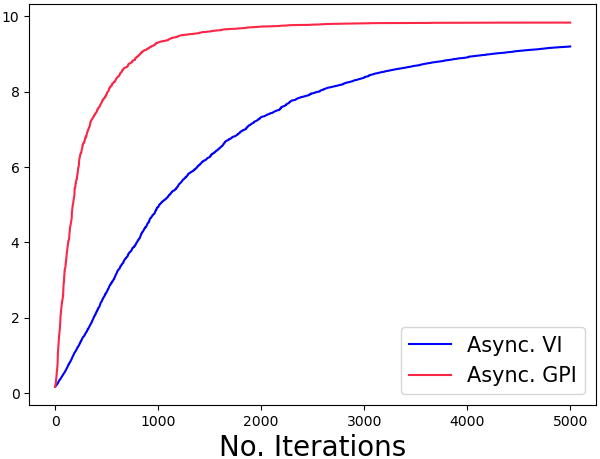}}
    \hfil
    \subfloat[]{\includegraphics[width=0.25\linewidth, height=3.2cm]{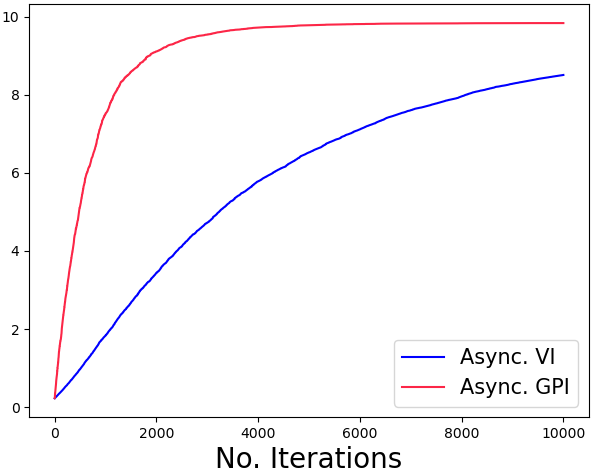}}
    \hfil
    \subfloat[]{\includegraphics[width=0.25\linewidth, height=3.2cm]{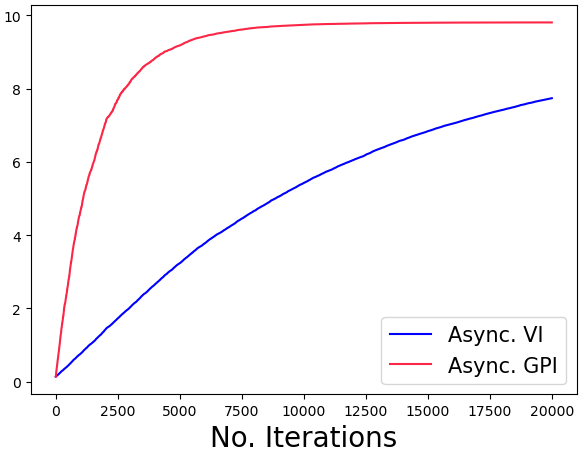}}
    \caption{Comparison between asynchronous geometric policy iteration (red curve) and asynchronous value iteration (blue curve) in 4 MDPs. $|\actions|=100$ for all MDPs and $|\states|=\{300, 500, 1000, 2000\}$ for (a)-(d), respectively. The horizontal axes are the number of updates. The vertical axes show the mean of the value function.}
    \label{fig:async}
\end{figure*}

\section{Experiments}
\label{sec:gpi_experiments}

We test GPI on random MDPs of different sizes. The baselines are policy iteration~(PI) and simple policy iteration~(SPI). We compare the number of iterations, actions switches, and wall time. Here we denote the number of iterations as the number of sweeps over the entire state set. Action switches are those policy updates within each iteration. The results are shown in Figure~\ref{fig:results_1}. We generate MDPs with $|\states| = \{100, 200, 300, 500, 1000\}$ corresponding to Figure~\ref{fig:results_1} (a)-(e). For each state size, we increase the number of actions (horizontal axes) to observe the difference in performance. The rows from the top to bottom are the number of iterations, action switches and wall time (vertical axes), respectively. Since SPI only performs one action switch per iteration, we only show its number of action switches. The purpose of adding SPI to the baseline is to verify if our GPI can effectively reduce the number of action switches. Since SPI sweeps over the entire state set and updates a single state with the largest improvement, it is supposed to have the least number of action switches. However, SPI's larger complexity of performing one update should lead to higher running time. This is supported by the experiments as Figure~\ref{fig:results_1} (a) and (b) show that SPI (green curves) takes the least number of switches and longest time. We drop SPI in Figure~\ref{fig:results_1} (c)-(e) to have a clearer comparison between GPI and PI (especially in wall time). 
The proposed GPI has a clear advantage over PI in almost all tests. The second row of Figure~\ref{fig:results_1} (a) and (b) shows that the number of action switches of GPI is significantly fewer than PI and very close to SPI although the complexity of a switch is cheaper by a factor of $|\states|$. And the reduction in the number of action switches leads to fewer iterations. Another important observation is that the margin increases as the action set becomes larger. This is strong empirical evidence that demonstrates the benefits of GPI's action selection strategy which is to reach the endpoints of line segments in the value function polytope. The larger the action set is, the more policies lying on the line segments and thus the more actions being excluded in one switch. The wall time of GPI is also very competitive compared to PI which further demonstrates that GPI can be a very practical algorithm for solving finite discounted MDPs.

We also test the performance of the asynchronous GPI (Async-GPI) on MDPs with $|\states| = \{300, 500, 1000, 2000\}$ and $|\actions| = 100$. For each setting, we randomly generate a sequence of states that is larger than $|\states|$. We compare Async-GPI with asynchronous value iteration (Async-VI) which is classic asynchronous dynamic programming algorithm. At time step $t$, Async-VI performs one step of the optimality Bellman operator on a single state $s_t$ that is available to the algorithm. The results are shown in Figure~\ref{fig:async}. The mean of the value function is plotted against the number of updates. We observe that Async-GPI took significantly fewer updates to reach the optimal value function. The gap becomes larger when the state set grows in size. These results are expected because Async-GPI also has a higher complexity to perform an update and Async-VI never really solves the real value function before reaching the optimality. 

%%%%%%%%%%%%%%%%%%%%%%%%%%%%%%%%%%%%%%%%%%%%%%%%%%%%%%
\section{Conclusions and Future Work}
In this paper, we discussed the geometric properties of finite MDPs. We characterized the hyperplane arrangement that includes the boundary of the value function polytope, and further related it to the MDP-LP polytope by showing that they share the same hyperplane arrangement. Unlike the well-defined MDP-LP polytope, it remains unclear which bounding hyperplanes are active and which halfspaces of them belong to the value space. Besides the conjecture stated earlier, we would like to understand in the future which cells of the hyperplane arrangement form the value function polytope, and may derive a bound on the number of convex cells. It is also plausible that the rest of the hyperplane arrangement will help us devise new algorithms for solving MDPs. 

Following the fact that policies that differ in only one state are mapped onto a line segment in the value function polytope, and that the only two policies on the polytope boundary are deterministic in that state, we proposed a new algorithm called geometric policy iteration that guarantees to reach an endpoint of the line segment for every action switch. We developed a mechanism that makes the value function monotonically increase with respect to action switches and the whole process can be computed efficiently. Our experiments showed that our algorithm is very competitive compared to the widely-used policy iteration and value iteration. We believe this type of algorithm can be extended to multi-agent settings, e.g., stochastic games~\cite{shapley1953stochastic}. It will also be interesting to apply similar ideas to model-based reinforcement learning. 

\begin{acks}
This work is supported by NSF DMS award 1818969 and a seed award from Center for Data Science and Artificial Intelligence Research at UC Davis. 
\end{acks}

% \clearpage
\bibliographystyle{ACM-Reference-Format}
\balance
\bibliography{sample-base}

\end{document}

%% file: dfn.tex
\newcommand{\excise}[1]{}%{$\star$\textsc{#1}$\star$}

\newtheorem{theorem}{Theorem}
\newtheorem{lemma}{Lemma}
\newtheorem{proposition}{Proposition}
{\bfseries}{\normalfont}
{\bfseries}{\normalfont}
\theoremstyle{definition}
\newtheorem{definition}{Definition}[section]

\newcommand{\valuefunction}{f_v}
\newcommand{\reals}{\mathbb{R}}

\newcommand{\bellop}{\mathcal{T}}

\newcommand{\affinesev}{H_{s_1,..,s_k}^{\pi}}

\newcommand{\agree}{Y^{\pi}_{s_1,..,s_k}}

\newcommand{\policies}{\mathcal{P}(\actions)^\states}

\newcommand{\gvec}{\succcurlyeq}

  {\begin{enumerate}%
    \setlength{\itemsep}{2pt}%
    \setlength{\parskip}{0pt}}%
  {\end{enumerate}}

\newcommand{\cbit}{\begin{compactitem}}
	\newcommand{\ceit}{\end{compactitem}}
\newcommand{\cben}{\begin{compactenum}}
	\newcommand{\ceen}{\end{compactenum}}

\newcommand{\bal}{\begin{align}}
\newcommand{\ean}{\end{align}}
\newcommand{\bit}{\begin{itemize}}
\newcommand{\eit}{\end{itemize}}
\newcommand{\ben}{\begin{enumerate}}
\newcommand{\een}{\end{enumerate}}
\newcommand{\beq}{\begin{equation}}
\newcommand{\eeq}{\end{equation}}

\newcommand{\bQ}{\mathbf{Q}}

\newcommand{\bw}{\mathbf{w}}

\newcommand{\bwiT}{\mathbf{w}_i^{\top}}

\newcommand{\bq}{\mathbf{q}}
\newcommand{\bone}{\mathbf{1}}

\newcommand{\bqi}{\mathbf{q}_i}

\newcommand{\hide}[1]{}

\newcommand{\rewards}{\mathcal{R}}
\newcommand{\transitions}{\mathcal{P}}
\newcommand{\actions}{\mathcal{A}}
\newcommand{\states}{\mathcal{S}}
\newcommand{\realset}{\mathbb{R}}
\newcommand{\E}{\mathop{\mathbb{E}}}
\newcommand{\cbar}{\, | \,}
\def \Ppi {P^\pi}
\newcommand{\Pdelta}{P^\delta}
\def \rpi {r^\pi}
\newcommand{\rdelta}{r^\delta}
\def \Vpi {V^\pi}
\newcommand{\policyspace}{\mathcal{P}(\actions)^\states}
\newcommand{\valuespace}{\mathcal{V}}

\newcommand{\Vpiprime}{V^{\pi'}}
\newcommand{\Vpik}{V^{\pi^{(k)}}}

\newcommand{\Vstar}{V^*}

\newcommand{\Ppiprime}{P^{\pi'}}

\newcommand{\Pstar}{P^*}
\newcommand{\rpiprime}{r^{\pi'}}

\newcommand{\Tstar}{\mathcal{T}^*}

\newcommand{\Dk}{\Delta^{*}_k}
\newcommand{\Dkmo}{\Delta^{*}_{k-1}}
\newcommand{\Dz}{\Delta^{*}_0}

\DeclareMathOperator*{\argmax}{\arg\!\max}
\newcommand{\norm}[1]{\left\lVert#1\right\rVert}
\newcommand{\norminf}[1]{\norm#1_{\infty}}
\newcommand{\piki}[2]{\pi^{(#1)}_{#2}}
\newcommand{\bQki}[2]{\mathbf{Q}^{(#1)}_{#2}}

\newcommand{\Pki}[2]{P^{(#1)}_{#2}}
\newcommand{\Vki}[2]{V^{(#1)}_{#2}}
\newcommand{\Tki}[2]{\bellop^{(#1)}_{#2}}
\newcommand{\rki}[2]{r^{(#1)}_{#2}}

\newcommand{\Vdelta}{V^{\delta}}

\def \agreeone {Y^\pi_{\states \setminus \{s\}}}

\newcommand{\parentheses}[1]{\left( #1 \right)}
\newcommand{\bigO}[1]{\mathcal{O}\parentheses{#1}}